\documentclass[11pt]{article} 

\usepackage{geometry}
\usepackage{graphicx}
\graphicspath{{./images/}}
\usepackage{hyperref}

\newcommand{\Rspace}        {\mm{\mathbb{R}}}
\usepackage[english]{babel}
\usepackage{babelbib}
\usepackage[utf8]{inputenc}
\usepackage[T1]{fontenc}

\usepackage{amsmath,braket,mathrsfs,url,epsfig,amsfonts}
\usepackage{amssymb,enumitem,fancyhdr}
\usepackage{latexsym}
\usepackage{amsthm}
\usepackage{mathtools}
\usepackage{epigraph,subfig,emptypage}
\usepackage[all]{xy}
\usepackage{color}
\usepackage{algorithm}

\usepackage{algorithmicx}
\usepackage{algpseudocode}
\usepackage{epstopdf}
\usepackage{lineno}
\usepackage{bbm}

\usepackage[small,nohug,heads=vee]{diagrams}
\diagramstyle[labelstyle=\scriptstyle]

\theoremstyle{definition}
\newtheorem{defin}{Definition}

\theoremstyle{definition}

\theoremstyle{plain}
\newtheorem{thm}[defin]{Theorem}
\newtheorem{lemma}[defin]{Lemma}
\newtheorem{coroll}[defin]{Corollary}

\usepackage[dvipsnames]{xcolor}
\usepackage{color}

\newcommand {\mm}[1] {\ifmmode{#1}\else{\mbox{\(#1\)}}\fi}

\newcommand{\ignore}[1]{}

\def \R {{\mathbb R}}

\def \Dpo {\Delta^{(1)}}
\def \Dpt {\Delta^{(2)}}

\def \L {{\mathcal L}}
\def \X {{\mathcal{X}}}
\def \Y {{\mathcal{Y}}}
\def \filt {\X}

\newcommand{\norm}[1]{\left\lVert #1 \right\rVert_2}
\newcommand{\infnorm}[1]{\left\lVert #1 \right\rVert_\infty}

\newcommand{\vol}{\mathrm{vol}}
\newcommand{\eps}{\epsilon}
\newcommand{\boxset}{\mathcal{B}}

\begin{document}

\nocite{*}

\title{A Kernel for Multi-Parameter Persistent Homology}
\author{Ren\'{e} Corbet\thanks{Graz University of Technology, corbet@tugraz.at}, Ulderico Fugacci\thanks{Graz University of Technology, fugacci@tugraz.at}, Michael Kerber\thanks{Graz University of Technology, kerber@tugraz.at}, Claudia Landi\thanks{University of Modena and Reggio Emilia, claudia.landi@unimore.it}, Bei Wang\thanks{University of Utah, beiwang@sci.utah.edu}}
\date{}

\maketitle

\begin{abstract}
Topological data analysis and its main method, persistent homology, provide a toolkit for computing topological information of high-dimensional and noisy data sets.
Kernels for one-parameter persistent homology have been established to connect persistent homology with machine learning techniques with applicability on shape analysis, recognition and classification. 
We contribute a kernel construction for multi-parameter persistence by integrating a one-parameter kernel weighted along straight lines. 
We prove that our kernel is stable and efficiently computable, which establishes a theoretical connection between topological data analysis and machine learning for multivariate data analysis. 
\end{abstract}

\section{Introduction}
\label{sec:introduction}

Topological data analysis (TDA) is an active area in data science with a growing interest and notable successes in a number of applications in science and engineering~\cite{CarstensHoradam2013,ChanCarlssonRabadan2013,GiustiPastalkovaCurto2015,GuoBanerjee2016,HiraokaNakamuraHirata2016,LiChengGlicksberg2015,WangOmbaoChung2017,YooKimAhn2016}. 
TDA extracts in-depth geometric information in amorphous solids~\cite{HiraokaNakamuraHirata2016}, determines robust topological properties of evolution from genomic data sets~\cite{ChanCarlssonRabadan2013} and identifies distinct diabetes subgroups~\cite{LiChengGlicksberg2015} and a new subtype of breast cancer~\cite{NicolauaLevinebCarlsson2011} in high-dimensional clinical data sets, to name a few. 
In the context of shape analysis, TDA techniques have been used in the recognition, classification~\cite{TurnerMukherjeeBoyer2014, LiOvsjanikovChazal2014}, summarization~\cite{BiasottiFalcidienoSpagnuolo2000}, and clustering~\cite{SkrabaOvsjanikovChazal2010}  of 2D/3D shapes and surfaces. 
Oftentimes, such techniques capture and highlight structures in data that conventional techniques fail to treat~\cite{LiOvsjanikovChazal2014,SkrabaOvsjanikovChazal2010} or reveal properly~\cite{HiraokaNakamuraHirata2016}. 

TDA employs the mathematical notion of \emph{simplicial complexes}~\cite{Munkres1984} to encode higher order interactions in the system, and at its core uses the computational framework of \emph{persistent homology}~\cite{EdelsbrunnerLetscherZomorodian2002,  ZomorodianCarlsson2005, EdelsbrunnerHarer2008, EdelsbrunnerHarer2010, EdelsbrunnerMorozov2012} to extract multi-scale topological features of the data. 
In particular, TDA extracts a rich set of topological features from high-dimensional and noisy data sets that complement geometric and statistical features, which offers a different  perspective for machine learning. 
The question is, \emph{how can we establish and enrich the theoretical connections between TDA and machine learning}?

Informally, \emph{homology} was developed to classify topological spaces by examining their topological features such as connected components, tunnels, voids and holes of higher dimensions; \emph{persistent homology} studies homology of a data set at multiple scales. Such information is summarized by the \emph{persistence diagram}, a finite multi-set of points in the plane.
A persistence diagram yields a complete description of the topological properties of a data set, making it an attractive
tool to define features of data that take topology into consideration. 
Furthermore, a celebrated theorem of persistent homology is the \emph{stability}
of persistence diagrams~\cite{Cohen-SteinerEdelsbrunnerHarer2007} -- small changes in the data lead to  small changes of the corresponding diagrams, making it suitable for robust data analysis. 

However, interfacing persistence diagrams directly with machine learning poses technical difficulties, because persistence diagrams contain point sets in the plane that do not have the structure of an inner product, which allows length and angle to be measured. 
In other words, such diagrams lack a Hilbert space structure for kernel-based learning methods such as kernel SVMs or PCAs~\cite{Reininghaus2015}. 
Recent work proposes several variants of \emph{feature maps}~\cite{Bubenik2015,Reininghaus2015, KwittHuberNiethammer2015}
that transform persistence diagrams into $L^2$-functions over $\Rspace^2$. This idea immediately enables the application of topological features for kernel-based machine learning methods as establishing a kernel function implicitly defines a Hilbert space structure~\cite{Reininghaus2015}. 

A serious limit of standard persistent homology and its initial interfacing with machine learning~\cite{Bubenik2015,Reininghaus2015, KwittHuberNiethammer2015, KusanoFukumizuHiraoka2016, HoferKwittNiethammer2017} is the restriction
to only a single scale parameter, thereby confining its applicability to the univariate setting.
However, in many real-world applications, such as data acquisition and geometric modeling, we often encounter richer information described by multivariate data sets~\cite{Carlsson2009,CarlssonMemoli2010,ChazalCohen-SteinerGuibas2009}. 
Consider, for example, climate simulations where multiple physical parameters such as temperature and pressure are computed simultaneously; and we are interested in understanding the interplay between these parameters. 
Consider another example in multivariate shape analysis, various families of functions carry information about the geometry of 3D shape objects, such as mesh density, eccentricity~\cite{SinghMemoliCarlsson2007} or Heat Kernel Signature~\cite{SunOvsjanikovGuibas2009}; and we are interested in creating multivariate signatures of shapes from such functions.  
Unlike the univariate setting, very few topological tools exist for the study of multivariate data~\cite{EdelsbrunnerHarer2002,EdelsbrunnerHarerPatel2008,SinghMemoliCarlsson2007}, let alone the integration of multivariate topological features with machine learning. 

The active area of \emph{multi-parameter persistent homology}~\cite{Carlsson2009} studies the extension of persistence to two or more (independent) scale parameters.  A complete discrete invariant such as the persistence diagram does not exist for more than one parameter~\cite{Carlsson2009}. To gain partial information, it is common to study \emph{slices}, that is, one-dimensional affine subspaces where all parameters are connected by a linear equation.
In this paper, we establish, for the first time, a theoretical connection between topological features and machine learning algorithms via the kernel approach for multi-parameter persistent homology.
Such a theoretical underpinning is necessary for applications in multivariate data analysis.

\paragraph{Our contribution}
 We propose the first kernel construction
for multi-parameter persistent homology.
Our kernel is \emph{generic},  \emph{stable} and
can be \emph{approximated in polynomial time}.
For simplicity, we formulate all our results for the case of two
parameters, although they extend to more than two parameters.   

Our input is a data set that is filtered according to two
scale parameters and has a finite description size;
we call this a \emph{bi-filtration} and postpone its formal definition
to Section~\ref{sec:preliminaries}.
Our main contribution is the definition of a feature map that assigns to a bi-filtration $\X$ a function $\Phi_\X:\Dpt\to\Rspace$, where $\Dpt$ is a subset of $\Rspace^4$. 
Moreover, $\Phi_\X^2$ is integrable over $\Dpt$,
effectively including the space of bi-filtrations into the Hilbert space
$L^2(\Dpt)$. Therefore, based on the standard scalar product in $L^2(\Dpt)$, a $2$-parameter 
kernel is defined such that for two given bi-filtrations $\X$ and $\Y$ we have 
\begin{equation}
\label{eqn:kernel_def}
\langle \X, \Y \rangle_{\Phi} := \int_{\Dpt} \Phi_\X\Phi_\Y d\mu.
\end{equation}
We construct our feature map by interpreting a point of $\Dpt$
as a pair of (distinct) points in $\R^2$ that define a unique slice.
Along this slice, the data simplifies to a \emph{mono-filtration} (i.e., a filtration that depends on a single scale parameter),
and we can choose among a large class
of feature maps and kernel constructions 
of standard, one-parameter persistence.  
To make the feature map well-defined,
we restrict our attention to a finite rectangle $R$.

Our inclusion into a Hilbert space induces a distance between bi-filtrations as
\begin{equation}
d_\Phi(\X,\Y):=\sqrt{\int (\Phi_\X- \Phi_\Y)^2 d\mu}.
\end{equation}
We prove a stability bound, relating this distance measure to the matching distance and the interleaving distance (see the paragraph on related work below).
We also show that this stability bound is tight up to constant factors (see Section~\ref{sec:stability}).

Finally, we prove that our kernel construction admits an efficient
approximation scheme. Fixing an absolute error bound $\eps$, we give a
polynomial time algorithm in $1/\eps$ and the  size of the bi-filtrations $\X$ and $\Y$ to compute a value $r$
such that $r\leq \langle \X, \Y\rangle_{\Phi}\leq r+\eps$.
On a high level, the algorithm subdivides the domain into boxes
of smaller and smaller width and
evaluates the integral of (\ref{eqn:kernel_def})
by lower and upper sums within each subdomain, terminating the process
when the desired accuracy has been achieved.
The technical difficulty lies in the accurate
and certifiable approximation of the variation of the feature map
when moving the argument within a subdomain.

\paragraph{Related work}
Our approach heavily relies on the construction of stable and efficiently
computable feature maps for mono-filtrations. This line of research was 
started by Reininghaus et al.~\cite{Reininghaus2015}, whose approach we
discuss in some detail in Section~\ref{sec:preliminaries}.
Alternative kernel constructions appeared in~\cite{KusanoFukumizuHiraoka2016,CarriereCuturiOudot2017}.
Kernel constructions fit into the general framework of including the space
of persistence diagrams in a larger space with more favorable properties.
Other examples of this idea are
persistent landscapes~\cite{Bubenik2015}
and persistent images~\cite{AdamsEmersonKirby2017},
which can be interpreted as kernel constructions as well. 
Kernels and related variants defined on mono-filtrations have been used to discriminate and classify shapes and surfaces~\cite{Reininghaus2015,HoferKwittNiethammer2017}. 
An alternative approach comes from the definition of suitable
(polynomial) functions on persistence diagrams to arrive at a fixed-dimensional
vector in $\R^d$ on which machine learning tasks can be performed;
see~\cite{AdcockRubinCarlsson2014, DiFabio2015, Adcock2016, Kalivsnik2018}.

As previously mentioned, a persistence diagram for multi-parameter persistence
does not exist~\cite{Carlsson2009}. However, bi-filtrations still admit
meaningful distance measures, which lead to the notion of closeness
of two bi-filtrations. The most prominent such distance is the 
\emph{interleaving distance}~\cite{Lesnick2015}, which, however, has recently been proved to be
NP-complete to compute and approximate~\cite{BjerkevikBotnanKerber2018}.
Computationally attractive alternatives are (multi-parameter) bottleneck distance \cite{DeyXin2018} and the \emph{matching distance}
\cite{BiasottiCerriFrosini2011,KerberLesnickOudot2019}, which compares the persistence diagrams along all slices
(appropriately weighted) and picks the worst discrepancy as the distance
of the bi-filtrations. This distance can be approximated up to a precision $\eps$
using an appropriate subsample of the lines \cite{BiasottiCerriFrosini2011}, and also computed exactly in polynomial time \cite{KerberLesnickOudot2019}. Our approach  extends these works in the sense that not just a distance,
but an inner product on bi-filtrations, is defined with our inclusion into
a Hilbert space.
In a similar spirit, the software library RIVET~\cite{Lesnick2015b}
provides a visualization tool to explore bi-filtrations by scanning
through the slices.

\section{Preliminaries}
\label{sec:preliminaries}

We introduce the basic topological terminology needed in this work.
We restrict ourselves to the case of simplicial complexes as input structures for a clearer geometric intuition of the concepts, but our results generalize to more abstract input types (such as minimal representations of persistence modules) without problems.

\paragraph{Mono-filtrations}
Given a vertex set $V$, an \emph{(abstract) simplex}
is a non-empty subset of $V$,
and an \emph{(abstract) simplicial complex} is a collection of such subsets
that is closed under the operation of taking non-empty subsets. 
A \emph{subcomplex} of a simplicial complex $X$
is a simplicial complex $Y$ with $Y\subseteq X$.
Fixing a finite simplicial complex $X$,
a \emph{mono-filtration} $\X$ of $X$ is a map that assigns
to each real number $\alpha$, 
a subcomplex $\X(\alpha)$ of $X$, with the property that
whenever $\alpha\leq\beta$, $\X(\alpha)\subseteq \X(\beta)$.
The \emph{size} of $\X$ is the number of simplices of $X$.
Since $X$ is finite, $\X(\alpha)$  changes  at only finitely many places
when $\alpha$ grows continuously from $-\infty$ to $+\infty$;
we call these values \emph{critical}.
More formally, $\alpha$ is \emph{critical} if there exists no open neighborhood
of $\alpha$ such that the mono-filtration assigns the identical subcomplex
to each value in the neighborhood.
For a simplex $\sigma$ of $X$,
we call the \emph{critical value} of $\sigma$ the infimum over all $\alpha$
for which $\sigma\in \X(\alpha)$. For simplicity, we assume that this infimum
is a minimum, so every simplex has a unique critical value wherever it is
included in the mono-filtration. 

\paragraph{Bi-filtrations}
For points in $\R^2$, we write $(a,b)\leq (c,d)$ if $a\leq c$ and $b\leq d$.
Similarly, we say $(a,b)<(c,d)$ if $a< c$ and $b< d$.
For a finite simplicial complex $X$, a \emph{bi-filtration} $\X$ of $X$
is a map that assigns to each point $p\in\R^2$ a subcomplex $\X(p)$ of $X$, 
such that whenever $p\leq q$, $\X(p)\subseteq \X(q)$.
Again, a point $p=(p_1,p_2)$ is called \emph{critical} for $\X$ if, for any $\epsilon>0$, both $\X(p_1-\epsilon, p_2)$ and $\X(p_1, p_2-\epsilon)$ are not identical to $\X(p)$.
Note that unlike in the mono-filtration case, the set of critical points
might not be finite. We call a bi-filtration \emph{tame} if it has only
 finitely many such critical points.
For a simplex $\sigma$, a point $p\in\R^2$ is \emph{critical} for $\sigma$
if, for any $\epsilon>0$, $\sigma$ is neither in $\X(p_1-\epsilon, p_2)$ nor in $\X(p_1, p_2-\epsilon)$, whereas $\sigma$ is in both $\X(p_1+\epsilon, p_2)$ and $\X(p_1, p_2+\epsilon)$.
Again, for simplicity, we assume that $\sigma\in \X(p)$ in this case.
A consequence of tameness is that each simplex has a finite number of critical points.
Therefore, we can represent a tame bi-filtration of a finite simplicial complex
$X$ by specifying the set of critical points for each simplex in $X$.
The sum of the number of critical points over all simplices of $X$ is called the \emph{size}
of the bi-filtration. We henceforth assume that bi-filtrations are always
represented in this form; in particular, we assume tameness throughout
this paper.

A standard example to generate bi-filtrations is by an arbitrary function
$F:X\to\R^2$ with the property that if $\tau\subset\sigma$ are two simplices
of $X$, $F(\tau)\leq F(\sigma)$. 
We define the \emph{sublevel set} $\X^F(p)$ as 
\begin{equation*}
\X^F(p):=\{\sigma\in X\mid F(\sigma)\leq p\},
\end{equation*}
and let $\X^F$ denote its corresponding \emph{sublevel set bi-filtration}. 
It is easy to verify that $\X^F$ yields a (tame) bi-filtration and $F(\sigma)$ is the unique critical value of $\sigma$ in the bi-filtration.

\begin{figure}[!h]
\centering 
\includegraphics[width=0.4\columnwidth]{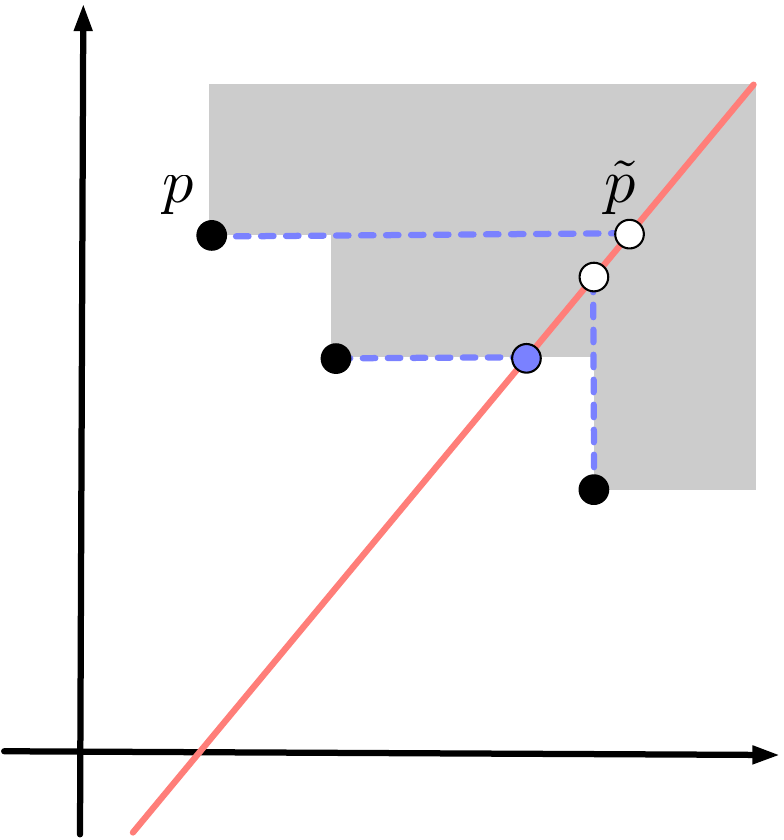}
\caption{The three black points mark the three critical points of some simplex $\sigma$ in $X$. The shaded area denotes the positions at which $\sigma$ is present in the bi-filtration. Along the given slice (red line), the dashed lines denote the first position
where the corresponding critical point ``affects'' the slice.
This position is either the upper-vertical, or right-horizontal projection of the critical point onto the slice, depending on whether the critical point is below or above
the line. For $\sigma$, we see that it enters the slice at the position marked
by the blue point.}
\label{fig:slice}
\end{figure}

\paragraph{Slices of a bi-filtration}
A bi-filtration $\X$ contains an infinite collection of mono-filtrations.
Let $\L$ be the set of all non-vertical lines in $\R^2$ with positive slope.
Fixing any line $\ell \in \L$,
we observe that when traversing this
line in positive direction, the subcomplexes of the bi-filtration are
nested in each other.
Note that $\ell$ intersects the anti-diagonal $x=-y$ in a unique base point $b$.
Parameterizing $\ell$ as $b+\lambda\cdot a$, where $a$ is the (positive)
unit direction vector of $\ell$, we obtain the mono-filtration
\[\X_\ell(\alpha):=\X(b+\alpha\cdot a).\]
We will refer to this mono-filtration $\X_\ell$ as a \emph{slice} of $\X$ along $\ell$
(and sometimes also call $\ell$ itself the slice, abusing notation).
The critical values of a slice can be inferred by the critical points 
of the bi-filtration in a computationally straightforward way.
Instead of a formal description, we refer to Figure~\ref{fig:slice} for a graphical description. Also, if the bi-filtration is of size $n$, each of its
slices is of size at most $n$.

\paragraph{Persistent homology}
A mono-filtration $\X$ gives rise to a persistence diagram. 
Formally, we obtain this diagram by applying the homology functor to $\X$,
yielding a sequence of vector spaces and linear maps between them, and splitting this sequence into indecomposable parts using representation theory.
Instead of rolling out the entire theory (which is explained, for instance, in~\cite{Oudot2015}), we give an intuitive description here. 

Persistent homology measures how the topological features of a data set evolve when considered across a varying scale parameter $\alpha$.
The most common example involves a point cloud in $\Rspace^d$, where considering a fixed scale $\alpha$ means replacing the points by balls of radius $\alpha$. 
As $\alpha$ increases, the data set undergoes various topological configurations, starting as a disconnected point cloud for $\alpha=0$ and ending up as a topological ball when $\alpha$ approaches $\infty$; see Figure~\ref{fig:persistence}(a) for an example in $\Rspace^2$. 

The topological information of this process can be summarized as
a finite multi-set of points in the plane, called the \emph{persistence diagram}. Each point of the diagram corresponds to a topological feature (i.e., connected components, tunnels, voids, etc.), and its coordinates specify at which scales the feature appears and
disappears in the data. As illustrated in Figure~\ref{fig:persistence}(a), all five  (connected) components are born (i.e.,~appear) at $\alpha=0$. The green component dies (i.e., ~disappears) when it merges with the red component at $\alpha=2.5$; similarly, the orange, blue and pink components die at scales $3$, $3.2$ and $3.7$, respectively. The red component never dies as $\alpha$ goes to $\infty$. The $0$-dimensional persistence diagram is defined to have one point per component with birth and death value as its coordinates (Figure~\ref{fig:persistence}(c)). The \emph{persistence} of a feature is then merely its distance from the diagonal. While we focus on the components, the concept generalizes to higher dimensions, such as tunnels ($1$-dimensional homology) and voids ($2$-dimensional homology). For instance, in Figure~\ref{fig:persistence}(a),  a tunnel appears at $\alpha=4.2$ and disappears at $\alpha=5.6$, which gives rise to a purple point $(4.2, 5.6)$ in the $1$-dimensional persistence diagram (Figure~\ref{fig:persistence}(c)). 

\begin{figure*}[!h]
\centering 
\includegraphics[width=1.04\columnwidth]{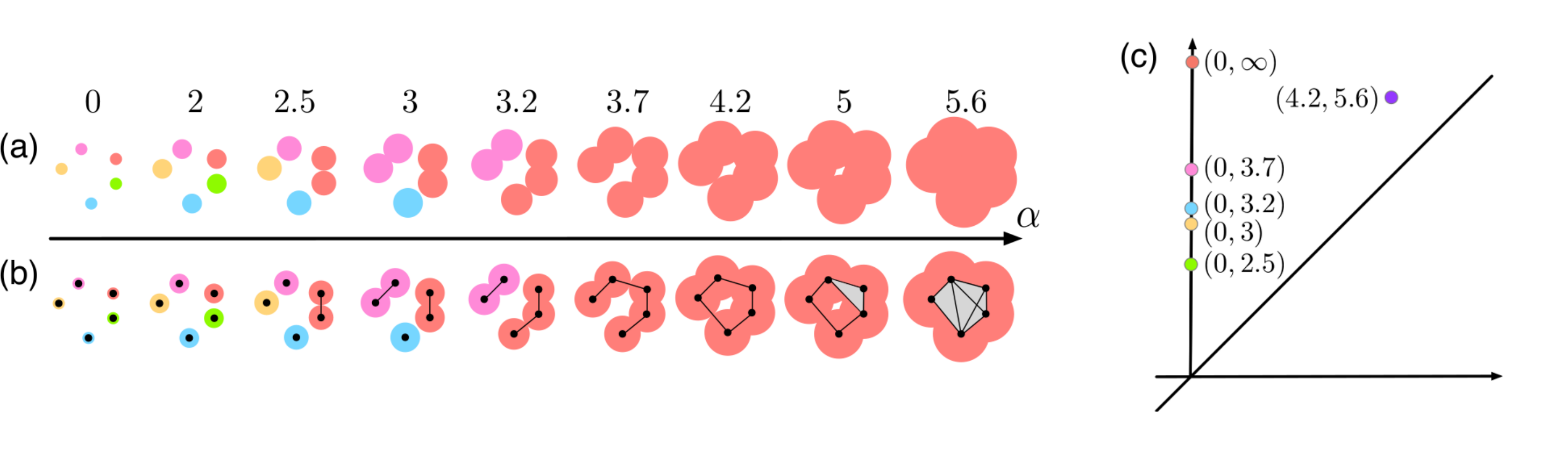}
\vspace{-8mm}
\caption{Computing persistent homology of a point cloud in $\Rspace^2$. (a) A nested sequence of topological spaces formed by unions of balls at increasing parameter values. (b) A mono-filtration of simplicial complexes that captures the same topological information as in (a). (c) $0$-dimensional and $1$-dimensional persistence diagrams combined.}
\label{fig:persistence}
\end{figure*}

 \begin{figure}
 \centering
 \vspace{-2mm}
 \includegraphics[width=0.58\linewidth]{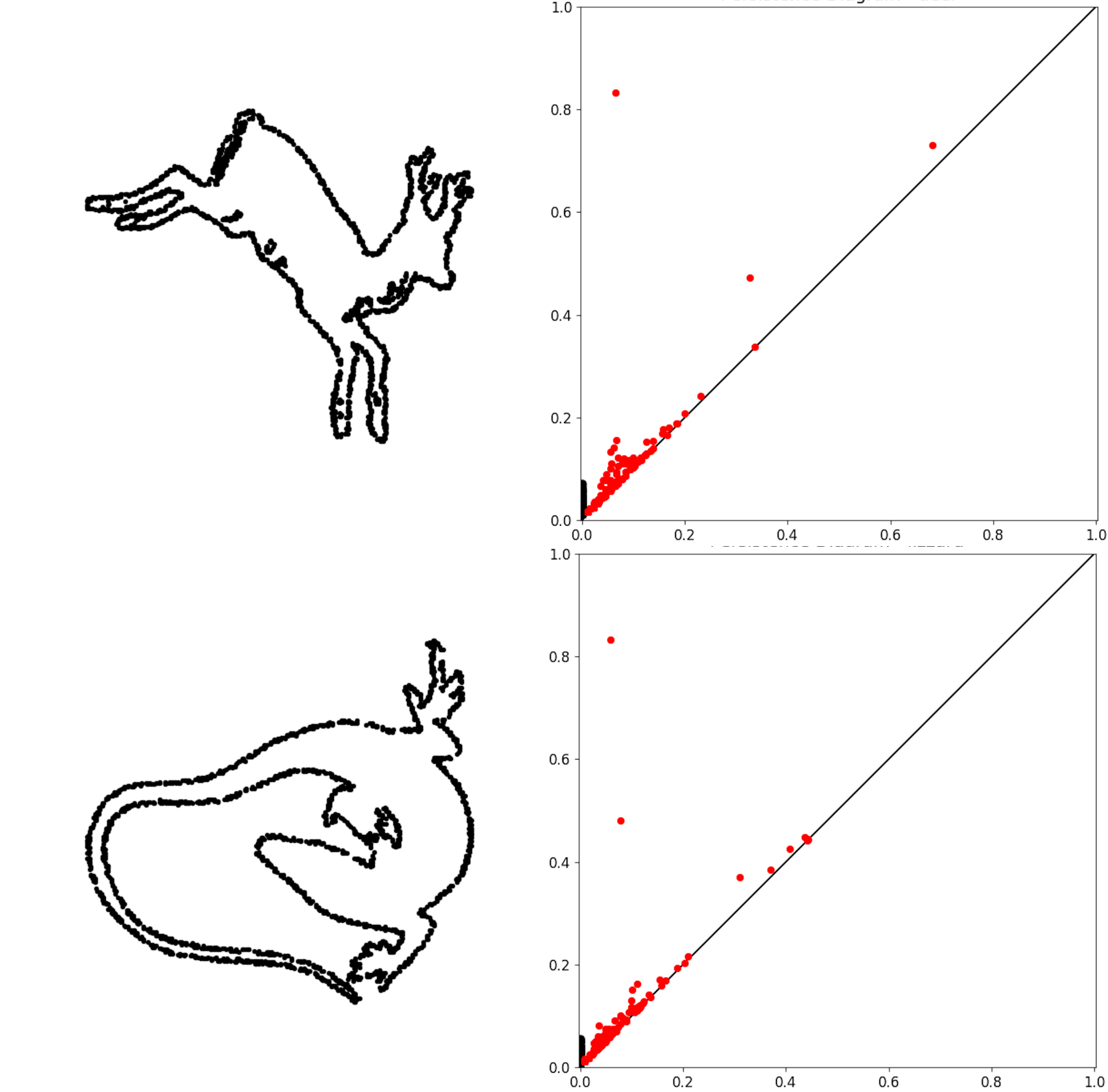}
 \vspace{-4mm}
 \caption{The persistence diagrams of 2D shape objects. Black and red points are $0$-dimensional and $1$-dimensional features respectively (ignoring points with $\infty$ persistence).}
 \label{fig:persistence-2D-objects}
 \end{figure}

From a computational point of view, the nested sequence of spaces formed by unions of balls (Figure~\ref{fig:persistence}(a)) can be replaced by a nested sequence of simplicial complexes by taking their nerves, thereby forming a mono-filtration of simplicial complexes that captures the same topological information but has a much smaller footprint (Figure~\ref{fig:persistence}(b)). 

In the context of shape analysis, we apply persistent homology to capture the topological information of 2D and 3D shape objects by employing various types of mono-filtrations. A simple example is illustrated in Figure~\ref{fig:persistence-2D-objects}: we extract point clouds sampled from the boundary of 2D shape objects and compute the persistence diagrams using Vietoris-Rips complex filtrations. 

\paragraph{Stability of persistent homology}
Bottleneck distance represents a similarity measure between persistence diagrams. Let $D$, $D'$ be two persistence diagrams. Without loss of generality, we can assume that both contain infinitely many copies of the points on the diagonal.
The {\em bottleneck distance} between $D$ and $D'$ is defined as
\begin{equation}
d_B(D, D'):= \inf_{\gamma} \sup_{x\in D} \|x-\gamma(x)\|_\infty,
\end{equation}
where $\gamma$ ranges over all bijections from $D$ to $D'$. We will also use the notation $d_B(\X,\Y)$ for two mono-filtrations
instead of $d_B(D(\X),D(\Y))$.

A crucial result for persistent homology is the {\em stability theorem} proven in \cite{Cohen-Steiner2007} and re-stated in our notation as follows. Given two functions $f,g: X\to \R$ whose sublevel sets form two mono-filtrations of a finite simplicial complex $X$, the induced persistence diagrams satisfy
\begin{equation}
d_B(D_f, D_g)\leq\|f-g\|_\infty:=\sup_{\sigma \in X} |f(\sigma)-g(\sigma)|.
\label{eqn:mono_stability}
\end{equation}

\paragraph{Feature maps for mono-filtrations}
Several feature maps aimed at the construction of a kernel for mono-filtrations have been proposed in the literature~\cite{Bubenik2015,KwittHuberNiethammer2015,Reininghaus2015}. 
We discuss one example: 
the persistence scale-space kernel~\cite{Reininghaus2015}
assigns to a mono-filtration $\X$ 
an $L^2$-function $\phi_\X$ defined on $\Dpo:=\left\{(x_1,x_2) \in \R^2 \mid x_1<x_2\right\}$.
The main idea behind the definition of $\phi_\X$ is to define a sum of Gaussian peaks, all of the same height and width, with each peak centered at one finite off-diagonal point of the persistence diagram $D(\X)$ of $\X$. To make the construction robust against perturbations, the function has to be equal to $0$ across the diagonal (the boundary of $\Dpo$).
This is achieved by adding negative Gaussian peaks at the reflections of 
the off-diagonal points along the diagonal. Writing $\bar{z}$ for the
reflection of a point $z$, we obtain the formula, 
\begin{equation}
\phi_\X(x):=\frac{1}{4\pi t} \sum_{z \in D(\X)} e^\frac{\|x-z\|_2^2}{4t} - e^\frac{\|x-\bar{z}\|_2^2}{4t},
\label{eqn:scale_space}
\end{equation}
where $t$ is the width of the Gaussian, which is a free parameter of
the construction. 
See Figure~\ref{fig:feature_map_illu} (b) and (c) for an illustration
of a transformation of a persistence diagram to the function $\phi_\X$. 
The induced kernel enjoys several stability properties
and can be evaluated efficiently without explicit construction of the feature
map; see~\cite{Reininghaus2015} for details.

More generally, in this paper, we look at the class of all feature maps
that assign to a mono-filtration $\X$ a function in $L^2(\Dpo)$.
For such a feature map $\phi_{\X}$, we define the following properties:
\begin{itemize}
\item {\em Absolutely boundedness.} There exists a constant $v_1>0$ such that, for any mono-filtration $\X$ of size $n$ and any $x \in \Dpo$, $0\leq\phi_\X(x)\leq v_1 \cdot n$.
\item {\em Lipschitzianity.} There exists a constant $v_2>0$ such that, for any mono-filtration $\X$ of size $n$ and any $x, x' \in \Dpo$, $|\phi_\X(x)- \phi_\X(x')|\leq v_2 \cdot n \cdot \|x-x'\|_2$.
\item {\em Internal stability.} There exists a constant $v_3>0$ such that, for any pair of mono-filtrations $\X, \Y$ of size $n$ and any $x \in \Dpo$, $|\phi_\X(x)- \phi_\Y(x)|\leq v_3 \cdot n \cdot d_B(\X, \Y)$.
\item {\em Efficiency.} For any $x \in \Dpo$, $\phi_\X(x)$ can be computed in polynomial time  in the size of $\X$,
that is, in $O(n^k)$ for some $k\geq 0$.
\end{itemize}
It can be verified easily that the scale-space feature map from above
satisfies all these properties. The same is true, for instance, if the Gaussian
peaks are replaced by linear peaks (that is, replacing the Gaussian kernel in (\ref{eqn:scale_space}) by a triangle kernel).

\section{A feature map for multi-parameter persistent homology}
\label{sec:definition}

Let $\phi$ be a feature map (such as the scale-space kernel) that assigns to a mono-filtration a function in $L^2(\Dpo)$. Starting from $\phi$, we construct a feature map $\Phi$ on the set of all bi-filtrations $\Omega$ that has values in a Hilbert space. 

The feature map $\Phi$ assigns to a bi-filtration $\filt$ a function $\Phi_\filt:\Dpt\to\R$.
We set
\[\Dpt:=\left\{(p,q)\mid p\in\R^2, q\in\R^2, p<q\right\}\]
as the set of all pairs of points where the first point is smaller than the second one.
$\Dpt$ can be interpreted naturally as a subset of $\R^4$, 
but we will usually consider elements of $\Dpt$ as pairs of points in $\R^2$.

Fixing $(p,q)\in\Dpt$, let $\ell$ denote the unique slice through these two points.
Along this slice, the bi-filtration gives rise to a mono-filtration $\filt_\ell$, and consequently
a function $\phi_{\filt_\ell}:\Dpo\to\R$ using the considered feature map for mono-filtrations.
Moreover, using the parameterization of the slice
$\ell$ as $b+\lambda\cdot a$ from Section~\ref{sec:preliminaries}, there exist
real values $\lambda_p,\lambda_q$ such that $b+\lambda_p a=p$ and $b+\lambda_q a=q$.
Since $p<q$ and $\lambda_p<\lambda_q$, hence $(\lambda_p,\lambda_q)\in\Dpo$.
We define $\Phi_\filt(p,q)$ to be the weighted function value of $\phi_{\filt_\ell}$ at $(\lambda_p,\lambda_q)$ (see also Figure~\ref{fig:feature_map_illu}), that is, 
\begin{equation}
\Phi_\filt(p,q):=w(p,q)\cdot\phi_{\filt_\ell}(\lambda_p,\lambda_q), 
\label{eqn:feature_map_def}
\end{equation}
where $w(p,q)$ is a weight function $w:\Dpt\to\R$ defined below.

The weight function $w$ has two components. First, let $R$ be a bounded axis-aligned rectangle in $\R^2$; its bottom-left corner coincides with the origin of the coordinate axes.
We define $w$ such that its weight is $0$ if $p$ or $q$ is outside of $R$.
Second, for pairs of points within $R\times R$, we assign a weight depending on the slope
of the induced slices. Formally, let $\ell$ be parameterized as $b+\lambda\cdot a$ as above, and recall that $a$ is
a unit vector with non-negative coordinates. Write $a=(a_1,a_2)$ and set $\hat{\ell}:=\min\{a_1,a_2\}$.
Then, we define
\[
w(p,q):= \chi_{R}(p) \cdot \chi_{R}(q) \cdot \hat{\ell}, 
\]
where $\chi_R$ is the characteristic function of $R$, mapping a point $x$ to $1$ if $x\in R$ and $0$ otherwise. 

The factor $\hat{\ell}$ ensures that slices that are close to being horizontal or vertical attain less importance in the feature map.
The same weight is assigned to slices in the matching distance~\cite{BiasottiCerriFrosini2011}.
$\hat{\ell}$ is not important for obtaining an $L^2$-function, 
but its meaning will become clear 
in the stability results of Section~\ref{sec:stability}.
We also remark that the largest weight is attained for the diagonal slice with a value of  $1/\sqrt{2}$. Consequently,
$w$ is a non-negative function upper bounded by $1/\sqrt{2}$. 

\begin{figure*}
 \begin{center}
\includegraphics[width=1.0\linewidth]{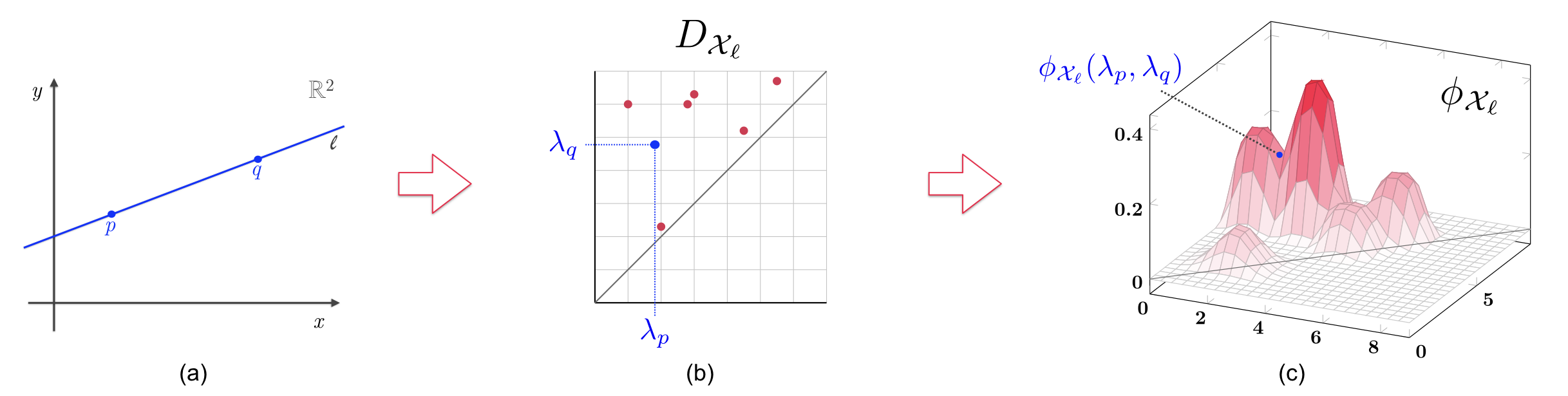}
 \end{center}
 \vspace{-6mm}
\caption{An illustration of the construction of a feature map for multi-parameter persistent homology. (a) Given a bi-filtration $\X$ and a point $(p,q)\in \Dpt$, the line $\ell$ passing through them is depicted and the parameter $\lambda_p$ and $\lambda_q$ computed.  (b) The point $(\lambda_p,\lambda_q)$ is embedded in the persistence diagram of the mono-filtration $\X_\ell$ obtained as the slice of $\X$ along $\ell$. (c) The point $(\lambda_p,\lambda_q)$ is assigned the value $\phi_{\X_\ell}(\lambda_p, \lambda_q)$ via the feature map $\phi$.}
\label{fig:feature_map_illu}
\end{figure*}

To summarize, our map $\Phi$ depends on the choice of an axis-aligned rectangle $R$ and a choice of feature map for mono-filtrations, which itself might
have associated parameters. For instance, using the scale-space feature map
requires the choice of the width $t$ (see (\ref{eqn:scale_space})). 
It is only left to argue that the image of the feature map $\Phi$ is
indeed an $L^2$-function. 

\begin{thm}\label{thm:well-defined}
If $\phi$ is absolutely bounded, then $\Phi_\X$ is in $L^2(\Dpt)$.
\end{thm}

\begin{proof}
Let $\filt$ be a bi-filtration of size $n$.
As mentioned earlier, each slice $\filt_\ell$ is of a size at most $n$.
By absolute boundedness and the fact that the weight function is upper bounded
by $\frac{1}{\sqrt{2}}$,
it follows that $|\Phi_\filt(p,q)|\leq \frac{v_1 n}{\sqrt{2}}$ for all $(p,q)$.
Since the support of $\Phi_\filt$ is compact ($R\times R$), 
the integral of $\Phi_\X^2$ over $\Dpt$ is finite, being absolutely bounded and
compactly supported.
\end{proof}

Note that Theorem~\ref{thm:well-defined} remains true even without restricting the weight function to $R$, provided we consider a weight function that is square-integrable over $\Dpt$. 
We skip the (easy) proof.

\section{Stability}
\label{sec:stability}

An important and desirable property for a kernel is its stability. In general, stability means that small perturbations in the input data imply small perturbations in the output data. In our setting, small changes between multi-filtrations (with respect to matching distance) should not induce large changes in their corresponding feature maps (with respect to $L^2$ distance). 

Adopted to our notation, the matching distance is defined as
\[
d_{match}(\X,\Y) = \sup_{\ell\in\L}\left( \hat{\ell}\cdot d_B(\X_\ell,\Y_\ell)\right),
\]
where $\L$ is the set of non-vertical lines with positive slope~\cite{Biasotti2008}.

\begin{thm}\label{thm:stability-match}
Let $\X$ and $\Y$ be two bi-filtrations.
If $\phi$ is absolutely bounded and internally stable, we have 
$$\|\Phi_{\X}-\Phi_{\Y}\|_{L^2}\leq C\cdot n\cdot\mathrm{area}(R)\cdot d_{match}(\X,\Y),$$
for some constant $C$.
\end{thm}

\begin{proof}
Absolute boundedness ensures that the left-hand side is well-defined
by Theorem~\ref{thm:well-defined}. Now we use the definition of $\|\cdot\|_{L^2}$ and the internal stability of $\phi$ to obtain
\begin{eqnarray*}
&\|\Phi_{\X}-\Phi_{\Y}\|_{L^2}^2\\
=& \underset{\Dpt}{\int} \left|w(p,q)\cdot\phi_{\X_{\ell}}(\lambda_p,\lambda_q)-w(p,q)\cdot\phi_{\Y_{\ell}}(\lambda_p,\lambda_q) \right|^2d\mu\\
\leq & \underset{\Dpt}{\int} \left( w(p,q)\cdot v_3 \cdot n\cdot d_B(\X_\ell,\Y_\ell)\right)^2d\mu\\
= & (v_3\cdot n)^2\underset{\Dpt}{\int} (w(p,q) \cdot d_B(\X_\ell,\Y_\ell))^2d\mu
\end{eqnarray*}
Since $w(p,q)$ is zero outside $R\times R$, the integral does not change when restricted to $\Dpt\cap(R\times R)$.
Within this set, $w(p,q)$ simplifies to $\hat{\ell}$, with $\ell$ the line through $p$ and $q$. Hence, we can further bound
\begin{eqnarray*}
= & (v_3\cdot n)^2\underset{\Dpt\cap (R\times R)}{\int} (\hat{\ell} \cdot d_B(\X_\ell,\Y_\ell))^2d\mu\\
\leq & (v_3\cdot n)^2\underset{\Dpt\cap (R\times R)}{\int} {\underbrace{\sup_{\ell\in\L} \left( \hat{\ell}\cdot d_B(\X_\ell,\Y_\ell)\right)}_{=d_{match}(\X,\Y)}}^2d\mu\\
= & \left(v_3\cdot n \cdot d_{match}(\X,\Y)\right)^2 \underset{\Dpt\cap (R\times R)}{\int}  1 d\mu.
\end{eqnarray*}
The claimed inequality follows by noting that the final integral is equal to $\frac{1}{4}\textnormal{area}(R)^2$.
\end{proof}

As a corollary, we get the the same stability statement with respect to interleaving distance instead of matching distance~\cite[Thm.1]{Landi2014}. 
Furthermore, we obtain a stability bound for sublevel set bi-filtrations
of functions $X\to\R^2$~\cite[Thm.4]{Biasotti2008}:

\begin{coroll}\label{cor:stability}
Let $F,G:X\to\R^2$ be two functions that give rise to sublevel set bi-filtrations $\X$ and $\Y$, respectively.
If $\phi$ is absolutely bounded and internally stable, we have
$$\|\Phi_{\X}-\Phi_{\Y}\|_{L^2}\leq C\cdot n\cdot\mathrm{area}(R)\cdot \|F-G\|_{\infty},$$
for some constant $C$.
\end{coroll}

We remark that the appearance of $n$ in the stability bound is not desirable as the bound
worsens when the complex size increases (unlike, for instance, the bottleneck stability bound
in (\ref{eqn:mono_stability}), which is independent of $n$). The factor of $n$ comes
from the internal stability property of $\phi$, so we have to strengthen this condition on $\phi$.
However, we show that such an improvement is impossible for a large class of ``reasonable''
feature maps.

For two bi-filtrations $\X,\Y$  we define  $\X\oplus \Y$ by setting $(\X\oplus\Y)(p):=\X(p)\sqcup\Y(p)$ for all $p\in\R^2$.
A feature map $\Phi$ is \emph{additive} if $\Phi_{\X\oplus\Y}=\Phi(\X)+\Phi(\Y)$ for all bi-filtrations $\X,\Y$.
$\Phi$ is called \emph{non-trivial} if there is a bi-filtration $\X$ such that $\|\Phi\|_{L^2}\neq 0$.
Additivity and non-triviality for feature maps $\phi$ on mono-filtrations is defined in the analogous way.
Note that, for instance, the scale space feature map is additive.
Moreover, because $(\X\oplus\Y)_{\ell}=\X_\ell\oplus\Y_\ell$ for every slice $\ell$,
a feature map $\Phi$ is additive if the underlying $\phi$ is additive.

For mono-filtrations, no additive, non-trivial feature map $\phi$ can satisfy
\[\|\phi_\X -\phi_\Y\|\leq C\cdot n^\delta\cdot d_B(\X,\Y)\]
with $\X,\Y$ mono-filtrations and $\delta\in[0,1)$;
the proof of this statement is implicit in~\cite[Thm 3]{Reininghaus2015}.
With similar ideas, we show that the same result holds in the multi-parameter case.

\begin{thm}\label{thm:additivity}
If $\Phi$ is additive and there exists $C>0$ and $\delta\in[0,1)$ such that 
$$\|\Phi_{\X}-\Phi_{\Y}\|_{L^2}\leq C\cdot n^\delta\cdot d_{match}(\X,\Y)$$
for all bi-filtrations $\X$ and $\Y$, then $\Phi$ is trivial.
\end{thm}

\begin{proof}
Assume to the contrary that there exists a bi-filtration $\X$ such that $\|\Phi_{\X}\|_{L^2}>0$. Then, writing $\mathcal{O}$ for the empty bi-filtration, by additivity we get $\|\Phi_{\sqcup_{i=1}^n\X}-\Phi_{\mathcal{O}}\|_{L^2}=n\|\Phi_{\X}-\Phi_{\mathcal{O}}\|_{L^2}>0$. On the other hand, $d_{match}(\sqcup_{i=1}^n\X,\mathcal{O})=d_{match}(\X,\mathcal{O})$.
Hence, with $C$ and $\delta$ as in the statement of the theorem, 
\begin{align*}
\frac{\|\Phi_{\sqcup_{i=1}^n\X}-\Phi_{\mathcal{O}}\|_{L^2}}{C\cdot n^\delta\cdot d_{match}(\sqcup_{i=1}^n\X,\mathcal{O})} &=\frac{n\|\Phi_{\X}-\Phi_{\mathcal{O}}\|_{L^2}}{C\cdot n^\delta\cdot d_{match}(\X,\mathcal{O})}\\
&=n^{1-\delta} \frac{\|\Phi_{\X}-\Phi_{\mathcal{O}}\|_{L^2}}{C\cdot d_{match}(\X,\mathcal{O}) }\stackrel{n\to\infty}{\longrightarrow} \infty,
\end{align*}
a contradiction.
\end{proof}

\section{Approximability}
\label{sec:approx}

We provide an approximation algorithm to compute
the kernel of two bi-filtrations $\X$ and $\Y$ up to any absolute error $\eps > 0$.
Recall that our feature map $\Phi$ depends on the choice
of a bounding box $R$. In this section, we assume $R$ 
to be the unit square $[0,1]\times [0,1]$ for simplicity.
We prove the following theorem that shows our kernel construction admits an efficient approximation scheme that is polynomial in $1/\eps$ and the size of the bi-filtrations. 

\begin{thm}
\label{thm:approx}
Assume $\phi$ is absolutely bounded, Lipschitz, internally stable
and efficiently computable.
Given two bi-filtrations $\X$ and $\Y$ of size $n$ and $\eps>0$, we can compute
a number $r$ such that $r\leq\langle \X,\Y\rangle_\Phi\leq r+\eps$
in polynomial time in $n$ and $1/\eps$.
\end{thm}

The proof of Theorem~\ref{thm:approx} will be illustrated in the following paragraphs, postponing most of the technical details to~\ref{sec:computability-appendix}.

\paragraph{Algorithm}
Given two bi-filtrations $\X$ and $\Y$ of size $n$ and $\eps>0$, our goal is to efficiently approximate $\langle \X,\Y\rangle_\Phi$ by some number $r$. 
On the highest level, we compute a sequence of approximation
intervals (with decreasing lengths) $J_1,J_2,J_3,\ldots$, each containing the desired kernel value $\langle \X,\Y\rangle_\Phi$.
The computation terminates as soon as we find some $J_i$ of width at most $\eps$,
in which case we return the left endpoint as an approximation to $r$.

For $s \in \mathbb{N}$ ($\mathbb{N}$ being the set of natural numbers), we compute $J_s$ as follows. We split $R$
into $2^s\times 2^s$ congruent squares (each of side length $2^{-s}$)
which we refer to as \emph{boxes}. 
See Figure~\ref{fig:computation_of_M}(a) for an example when $s=3$.
We call a pair of such boxes
a \emph{box pair}. The integral from (\ref{eqn:kernel_def})
can then be split into a sum of integrals over all $2^{4s}$ box pairs. 
That is, 
\begin{equation*}
\langle \X,\Y\rangle_\Phi = \int_{\Dpt} \Phi_\X\Phi_\Y d\mu
= \sum_{(B_1,B_2)}\int_{\Dpt \cap (B_1\times B_2)}\Phi_\X \Phi_\Y d\mu.
\end{equation*}
For each box pair, we compute an approximation interval for the integral,
and sum them up using interval arithmetic to obtain $J_s$.

\begin{figure}[!h]
\centering
\includegraphics[width=.7\linewidth]{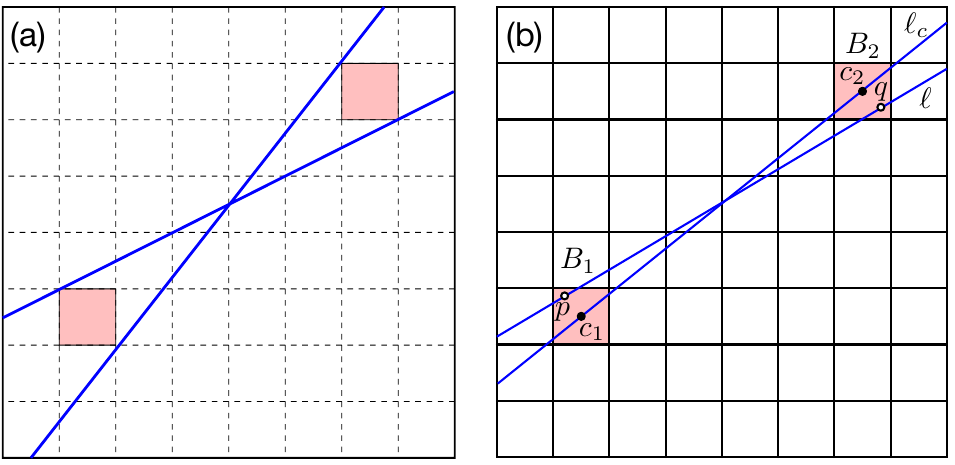}
\caption{(a) The two given slices realize the largest and smallest possible
slope among all slices traversing the pink box pair.  It can be easily seen
that the difference of the unit vector of the center line to one of the unit vectors of these two lines realizes $A$ for the given box pair. (b) Computing variations for the center slice and a traversing slice of a box pair.}
\label{fig:computation_of_M}
\end{figure}

We first give some (almost trivial) bounds for $\langle \X,\Y\rangle_\Phi$. 
Let $(B_1,B_2)$ be a box pair with centers located at $c_1$ and $c_2$, respectively. 
By construction, $\vol(B_1\times B_2) = 2^{-4s}$.
By the absolute boundedness of $\phi$, we have 
\begin{align}
\label{eqn:js}
 \int_{\Dpt \cap (B_1\times B_2)}\Phi_\X \Phi_\Y d\mu &
\le \int_{(B_1\times B_2)} \left(\frac{1}{\sqrt{2}} v_1 n \cdot \frac{1}{\sqrt{2}} v_1 n\right)  d\mu \\
 &= \frac{v_1^2 n^2}{2}\vol(B_1\times B_2) 
 =  \frac{v_1^2 n^2}{2^{4s+1}}, 
\end{align}
where $1/\sqrt{2}$ is the maximal weight. 
Let $U: = \frac{v_1^2 n^2}{2^{4s+1}}$. 
If $c_1 \leq c_2$, then we can choose $[0, U]$ as approximation interval.
Otherwise, if $c_1 \not\leq c_2$, then $\Dpt \cap (B_1\times B_2) = \emptyset$; 
we simply choose $[0,0]$ as approximation interval.

We can derive a second lower and upper bound for $\langle \X,\Y\rangle_\Phi$  as follows.  
We evaluate $\Phi_\X$ and $\Phi_\Y$ at the pair of centers $(c_1,c_2)$, which is possible due to the efficiency hypothesis of $\phi$.
Let $v_\X = \Phi_\X(c_1, c_2)$ and $v_\Y = \Phi_\Y(c_1, c_2)$. 
Then, we compute \emph{variations}
$\delta_\X,\delta_\Y \geq 0$ relative to the box pair, with the property that, 
for any pair $(p,q)\in B_1\times B_2$,
$\Phi_\X(p,q)\in[v_\X-\delta_\X,v_\X+\delta_\X]$,
and $\Phi_\Y(p,q)\in[v_\Y-\delta_\Y,v_\Y+\delta_\Y]$. 
In other words, variations describe how far the value of $\Phi_\X$ (or $\Phi_\Y$) deviates from its value at $(c_1,c_2)$ within $B_1 \times B_2$. 
Combined with the derivations starting in (\ref{eqn:js}), we have for any pair $(p,q)\in B_1\times B_2$, 
\begin{align}
\label{eq:var}
& \max\left\{0,(v_\X-\delta_\X)(v_\Y-\delta_\Y)\right\} \\
\leq & \Phi_\X(p,q)\Phi_\Y(p,q) \\
\leq & \min\left\{\frac{v_1^2 n^2}{2},(v_\X+\delta_\X)(v_\Y+\delta_\Y)\right\}.
\end{align}

By multiplying the bounds obtained in (\ref{eq:var}) by the volume of $\Dpt \cap (B_1\times B_2)$, we get a lower and an upper bound for the integral of $\Phi_\X \Phi_\Y$ over a box pair $(B_1,B_2)$. By summing over all possible box pairs, the obtained lower and upper bounds are the endpoints of $J_s$.

\paragraph{Variations}
We are left with computing the variations relative to a box pair.
For simplicity, we set $\delta:=\delta_\X$
and explain the procedure only for $\X$; the treatment of $\Y$ is similar.

We say that a slice $\ell$ \emph{traverses} $(B_1,B_2)$ if it intersects
both boxes in at least one point.
One such slice is the \emph{center slice}  ${\ell_c}$,
which is the slice through $c_1$ and $c_2$.
See Figure~\ref{fig:computation_of_M}(b) for an illustration. 
We set $D$ to be the maximal bottleneck distance of the center slice
and every other slice traversing the box pair (to be more precise, of
the persistence diagrams along the corresponding slices).
We set $W$ as the maximal difference between the weight of the center slice
and any other slice traversing the box pair, where the weight $w$ is
defined as in Section~\ref{sec:definition}.
Write $\lambda_{c_1}$ for the parameter value of $c_1$ along the center slice.
For every slice $\ell$ traversing the box pair and any point $p\in\ell\cap B_1$, we have a value $\lambda_p$, yielding the parameter value of $p$ along $\ell$.
We define $L_1$ as the maximal difference of $\lambda_p$ and $\lambda_{c_1}$
among all choices of $p$ and $\ell$. We define $L_2$ in the same way for
$B_2$ and set $L:=\max\{L_1,L_2\}$. With these notations, we obtain Lemma~\ref{lemma:up} below. 

\begin{lemma}
\label{lemma:up}
For all $(p,q)\in B_1\times B_2$,
\[\left|\Phi_\X(p,q)-\Phi_\X(c_1,c_2)\right|\leq \frac{v_3 n}{\sqrt{2}} D +v_1n W + v_2 n L.\]
\end{lemma}
\begin{proof}

Plugging in (\ref{eqn:feature_map_def}) and using triangle inequality, we obtain
\begin{align*}
&\left|\Phi_\X(p,q)-\Phi_\X(c_1,c_2)\right| \\
= & \left| \hat{\ell}\phi_{\X_{\ell}}(\lambda_p,\lambda_q) - \hat{{\ell_c}}\phi_{\X_{{\ell_c}}}(\lambda_{c_1},\lambda_{c_2})\right|\\
\leq& \hat{\ell}\left|\phi_{\X_{\ell}}(\lambda_p,\lambda_q)-\phi_{\X_{{\ell_c}}}(\lambda_p,\lambda_q)\right|
+\phi_{\X_{{\ell_c}}}(\lambda_p,\lambda_q)\left|\hat{\ell}-\hat{{\ell_c}}\right| \\
&+\hat{{\ell_c}}\left|\phi_{\X_{{\ell_c}}}(\lambda_p,\lambda_q)-\phi_{\X_{{\ell_c}}}(\lambda_{c_1},\lambda_{c_2})\right|
\end{align*}
and bound the three parts separately. The first summand is upper bounded
by $\frac{v_3 n D}{\sqrt{2}}$ because of internal stability of the feature map $\phi$ and because $\hat{\ell}\leq\frac{1}{\sqrt{2}}$
for any slice $\ell$. The second summand is upper bounded by
$v_1 n W$ by the absolute boundedness of $\phi$.
The third summand is bounded by $v_2 n L$, because
$\|(\lambda_p,\lambda_q)-(\lambda_{c_1},\lambda_{c_2})\|_2
\leq \sqrt{2}\|(\lambda_p,\lambda_q)-(\lambda_{c_1},\lambda_{c_2})\|_{\infty}
\leq L$ and by $\phi$ being Lipschitz, $\left|\phi_{\X_{{\ell_c}}}(\lambda_p,\lambda_q)-\phi_{\X_{{\ell_c}}}(\lambda_{c_1},\lambda_{c_2})\right|\leq \sqrt{2}v_2 n L$,
and $\hat{\ell}\leq\frac{1}{\sqrt{2}}$. The result follows.
\end{proof}

Next, we bound $D$ by simple geometric quantities.
We use the following lemma, whose proof appeared in~\cite{Landi2014}:

\begin{lemma}[\cite{Landi2014}]
\label{lemma:landi}
Let $\ell$ and $\ell'$ be two slices with parameterizations $b+\lambda a$ 
and $b'+\lambda a'$, respectively. Then, the bottleneck distance
of the two persistence diagrams along these slices is upper bounded by
\[\frac{2\|a-a'\|_\infty + \|b-b'\|_\infty}{\hat{\ell}\hat{\ell'}}.\]
\end{lemma}
We define $A$ as the maximal infinity distance of the directional vector
of the center slice $\ell_c$ and any other slice $\ell$ traversing the box pair.
We define $B$ as the maximal infinity distance of the base point
of $\ell_c$ and any other $\ell$. Finally, we set
$M$ as the minimal weight among all slices traversing the box pair. Using
Lemma~\ref{lemma:landi}, we see that
\[D\leq \frac{2A+B}{M \hat{{\ell_c}}},\]
and we set
\begin{align}
\delta:= \frac{v_3 n(2A+B)}{\sqrt{2}M\hat{{\ell_c}}}  +v_1n W + v_2 n L.
\label{eqn:variation_bound}
\end{align}
It follows from Lemma~\ref{lemma:up} and Lemma~\ref{lemma:landi} that $\delta$ indeed satisfies
the required variation property. 

We remark that $\delta$ might well be equal
to $\infty$, if the box pair admits a traversing slice that is horizontal
or vertical, in which case the lower and upper bounds
derived from the variation are vacuous.
While (\ref{eqn:variation_bound}) looks complicated, the
values $v_1,v_2,v_3$ are constants coming from the
considered feature map $\phi$, and all the remaining values 
can be computed in constant time using elementary geometric properties
of a box pair. We only explain the computation of $A$ in Figure~\ref{fig:computation_of_M}(a) and skip the details of the other values.

\paragraph{Analysis}
At this point, we have not made any claim that the algorithm is guaranteed to terminate.
However, its correctness follows at once because $J_s$ indeed
contains the desired kernel value. Moreover, handling one box pair
has a complexity that is polynomial in $n$,
because the dominant step is to evaluate $\Phi_\X$ at the center $(c_1,c_2)$.
Hence, if the algorithm terminates at iteration $s_0$, its complexity is
\[\sum_{s=1}^{s_0} O\left(2^{4s} poly(n)\right).\]
This is because in iteration $s$, $2^{4s}$ box pairs need to be considered.
Clearly, the geometric series above is dominated by the last iteration,
so the complexity of the method is $O(2^{4s_0} poly(n))$.
The last (and technically most challenging) step is to argue that
$s_0=O(\log n+\log\frac{1}{\eps})$, which implies that the algorithm indeed
terminates and its complexity
is polynomial in $n$ and $1/\eps$.

To see that we can achieve any desired accuracy for the value of the kernel,  i.e., that the interval width tends to 0, we observe that, if the two boxes $B_1$, $B_2$ are sufficiently far away and the resolution $s$ is sufficiently large, the magnitudes $A$, $B$, $W$, and $L$ in (\ref{eqn:variation_bound}) are all small, because the parameterizations of two slices traversing the box pair are similar (see Lemmas \ref{angle_lemma}, \ref{b_lemma}, \ref{weight_lemma} and \ref{lambda_lemma} in~\ref{sec:computability-appendix}). Moreover, if every slice traversing the box pair has a sufficiently large weight (i.e., the slice is close to the diagonal), the value $M$ in (\ref{eqn:variation_bound}) is sufficiently large. These two properties combined imply that the variation of such a box pair (which we refer to as the {\em good} type) tends to 0 as $s$ goes to $\infty$. Hence, the bound based on the variation tends to the correct value for good box pairs.

However, no matter how high the resolution, there are always \emph{bad} box pairs for which either $B_1$, $B_2$ are close, or are far but close to horizontal and vertical, and hence yield a very large variation. For each of these box pairs, the bounds derived from the variation are vacuous, but we still have the trivial bounds $[0, U]$ based on the absolute boundedness
of $\phi$. Moreover, the total volume of these bad box pairs goes to 0 when $s$ goes to $\infty$ (see Lemma \ref{close_lemma}, Lemma~\ref{non-diag_lemma} in~\ref{sec:computability-appendix}). So, the contribution of these box pairs tends to 0.
These two properties complete the proof of Theorem~\ref{thm:approx}.

A more careful investigation of our proof
shows that the complexity of our algorithm can be expressed as 
$O(n^{80+k}(1/{\eps})^{40}),$
where $k$ is the efficiency constant of the feature map
as defined at the end of Section~\ref{sec:preliminaries}.\footnote{
We made little effort to optimize the exponents in this bound.}

\section{Conclusions and future developments}
\label{sec:concl}

We restate our main results for the case of a multi-filtration $\X$
with $d$ parameters: there is a feature map that associates to $\filt$
a real-valued function $\Phi_\filt$ whose domain is of dimension $2d$,
and introduces a kernel between a pair of multi-filtrations with a stable distance function, where the stability bounds depend on the ($2d$-dimensional) volume of a chosen bounding box.
The proofs of these generalized results carry over from the results of this paper.
Moreover, assuming that $d$ is a constant, we claim that the kernel can be approximated
in polynomial time to any constant (with the polynomial exponent depending on $d$).
A proof of this statement requires to adapt the definitions and proofs of~\ref{sec:computability-appendix}
to the higher-dimensional case; we omit details.

Other generalizations include replacing filtrations of simplicial complexes with persistence modules
(with a suitable finiteness condition), passing to sublevel sets of a larger
class of (tame) functions and replacing the scale-space feature map 
with a more general family of single-parameter feature maps.
All these generalizations will be discussed in subsequent work.

The next step is an efficient implementation of our kernel approximation algorithm.
We have implemented a prototype in C++, realizing a more adaptive version of
the described algorithm. We have observed rather poor performance due to the sheer
number of box pairs to be considered. 
Some improvements under consideration are to precompute all
combinatorial persistence diagrams 
(cf.~the barcode templates from~\cite{Lesnick2015b}),
to refine the search space adaptively using a quad-tree
instead of doubling the resolution and to use techniques from numerical
integration to handle real-world data sizes. We hope that an efficient
implementation of our kernel will validate the assumption that including
more than a single parameter will attach more information to the data set 
and improve the quality of machine learning algorithms using topological
features.

\section*{Acknowledgments}
This work was initiated at the Dagstuhl Seminar 17292 ``Topology, Computation and Data Analysis". 
We thank all members of the breakout session on multi-dimensional
kernel for their valuable suggestions. 
The first three authors acknowledge support by the Austrian
Science Fund (FWF) grant P 29984-N35.
The last author acknowledges partial support by NSF grant DBI-1661375, IIS-1513616 and NIH grant R01-1R01EB022876-01. 

\bibliographystyle{abbrv}
\bibliography{arxiv}


\section*{Appendix}

\appendix

\section{Details on the Proof of Theorem~\ref{thm:approx}}
\label{sec:computability-appendix}

\paragraph{Overview}
Recall that our approximation algorithm produces an approximation
interval $J_s$ for $s\in\mathbb{N}$ by splitting the unit square
into $2^s\times 2^s$ boxes.
For notational convenience, we write $u:=2^{-s}$ for the side length
of these boxes.

We would like to argue that the algorithm terminates after $O(\log n + \log\frac{1}{\eps})$
iterations, which means that after that many iterations, an interval
of width $\eps$ has been produced. The following Lemma~\ref{lemma:eq} gives an equivalent criterion in terms of $u$ and~$n$.

\begin{lemma}
\label{lemma:eq}
Assume that there are constants $e_1,e_2>0,$ such that
$\mathrm{width}(J_s)=O(n^{e_1}u^{e_2})$. Then, $\mathrm{width}(J_{s_0})\leq\eps$ for some $s_0 = O\left(\log n+\log \frac{1}{\eps}\right)$.
\end{lemma}
\begin{proof}
Assume that $\mathrm{width}(J_s)\leq c n^{e_1}u^{e_2}$ for constants $c$ and $s$ sufficiently large. Since $u=2^{-s}$,
a simple calculation shows that $c n^{e_1}u^{e_2}\leq\eps$ if and only if $s\geq \frac{\log c+e_1\log n + \log \frac{1}{\eps}}{e_2}$.
Hence, choosing
\[s:=\left\lceil\frac{\log c+e_1\log n + \log \frac{1}{\eps}}{e_2} \right\rceil=O\left(\log n+\log\frac{1}{\eps}\right)\]
ensures that $\mathrm{width}(J_{s_0})\leq\eps$.
\end{proof}

In the rest of this section, we will show that $\mathrm{width}(J_s)=O(n^2 u^{0.1})$.

\paragraph{Classifying box pairs}
For the analysis, we partition the box pairs considered by the algorithm into $4$ disjoint classes.
We call a box pair $(B_1,B_2)$:
\begin{itemize}
\item {\em null} if $c_1\nleq c_2$,
\item {\em close} if $c_1\leq c_2$ such that $\|c_1-c_2\|_2<\sqrt{u}$,
\item {\em non-diagonal} if $c_1\leq c_2$ such that $\|c_1-c_2\|_2\geq\sqrt{u}$ and any line $\ell$ that traverses $(B_1,B_2)$ satisfies $\hat{\ell}<u^{\frac{1}{5}}$,
\item {\em good} if it is of neither of the previous three types.
\end{itemize}
According to this notation, the integral from (\ref{eqn:kernel_def}) can then be split as
\[\langle \X, \Y \rangle_{\Phi}= \langle \X, \Y \rangle_{null} + \langle \X, \Y \rangle_{close} + \langle \X, \Y \rangle_{non-diag} + \langle \X, \Y \rangle_{good},\]
where, $\langle \X, \Y \rangle_{null}$ is defined as
$\sum_{(B_1,B_2) \, null}\int_{\Dpt \cap (B_1\times B_2)}\Phi_\X \Phi_\Y d\mu$,
and analogously for the other ones.
We let $J_{s,null}$, $J_{s,close}$, $J_{s,non-diag}$, $J_{s,good}$ denote the four approximation intervals
obtained from our algorithm when summing up the contributions of the corresponding box pairs.
Then clearly, $J_s$ is the sum of these four intervals. For simplicity, we will write
$J_{null}$ instead of $J_{s,null}$ when $s$ is fixed, and likewise for the other three cases.

We observe first that the algorithm yields $J_{null}=[0,0]$, so null box pairs can simply be ignored.
Box pairs that are either close or non-diagonal are referred to as \emph{bad} box pairs in Section~\ref{sec:approx}.
We proceed by showing that the width of $J_{close}$, $J_{non-diag}$, and  $J_{good}$
are all bounded by $O(n^2 u^{0.1})$.

\paragraph{Bad box pairs}
We start with bounding the width of $J_{close}$. 
Let $\boxset_{\text{close}}$ be the union of all close box pairs.
Note that our algorithm assigns to each box pair
$(B_1,B_2)$
an interval that is a subset of $[0,U]$. Recall that $U=\frac{v_1^2 n^2}{2^{4s+1}}$.  $U$ can be rewritten as $\frac{v_1^2 n^2}{2}\vol(B_1 \times B_2)$,
where $\vol(B_1 \times B_2)$ is the $4$-dimensional volume of the box pair $(B_1, B_2)$.
It follows that
\begin{equation}
\label{eqn:width}
\mathrm{width}(J_{close})\leq \frac{v_1^2 n^2}{2}\vol(\boxset_{\text{close}}).
\end{equation}

\begin{lemma}
\label{close_lemma}
For $u\leq \frac{1}{2}$, $\vol(\boxset_{\text{close}})\leq 4\pi u $.
\end{lemma}

\begin{proof}
Fixed a point $p\in R$, for each point $q\in R$ such that $(p,q)\in \boxset_{\text{close}}$ and $p<q$, there exists a unique close box pair $(B_1, B_2)$ that contains $(p,q)$.
By definition of close box pair, we have that:
\begin{align*}
\norm{p-q}\leq \norm{p-c_1} + \norm{c_1-c_2} + \norm{c_2-q} \leq \sqrt{u}+\sqrt{2}u.
\end{align*}
Moreover, for $u\leq \frac{1}{2}$, $\sqrt{2}u\leq \sqrt{u}$, and so $\norm{p-q}\leq 2 \sqrt{u}$.
Equivalently, $q$ belongs to the 2-ball $B(p,2\sqrt{u})$ centered at $p$ and of radius $2\sqrt{u}$. Then,
\begin{align*}
\vol(\boxset_{\text{close}})&=\int_{\boxset_{\text{close}}} 1 d\mu \leq\int_{p\in R} \int_{q\in B(p,2\sqrt{u})} 1 d\mu \\
&\leq\int_{p\in R} 4\pi u d\mu = 4\pi u. \qedhere
\end{align*}
\end{proof}

Consequently, combined with (\ref{eqn:width}), we have 
\begin{equation*}
\mathrm{width}(J_{close})\leq \frac{4\pi v_1^2 n^2}{2} u = O(n^2 u^{0.1}).
\end{equation*}
Note that $u<1$ and hence, $u\leq u^{0.1}$.

\medskip

For the width of $J_{non-diag}$, we use exactly the same reasoning,
making use of the following Lemma~\ref{non-diag_lemma}. 
Let $\boxset_{\text{non-diag}}$ be the union of all non-diagonal box pairs.

\begin{lemma}
\label{non-diag_lemma}
For $u\leq 2^{-\frac{5}{2}}$, $\vol(\boxset_{\text{non-diag}})\leq \sqrt{2} u^{\frac{1}{5}}$.
\end{lemma}
\begin{proof}
Fixed a point $p\in R$, for each point $q\in R$ such that $(p,q)\in \boxset_{\text{non-diag}}$ and $p<q$, there exists a unique non-diagonal box pair $(B_1, B_2)$ that contains $(p,q)$.
We have that $q$ lies in:
\begin{itemize}
\item Triangle $T_1(p)$ of vertices $p=(p_1, p_2)$, $(1, p_2)$, and $(1, p_2 + (1-p_1)\frac{a_2}{a_1})$, if the line $\ell$ of maximum slope passing through $B_1 \times B_2$ is such that $\hat{\ell}=a_2$ where $a=(a_1,a_2)$ is the (positive) unit direction vector of $\ell$;
\item Triangle $T_2(p)$ of vertices $p=(p_1, p_2)$, $(p_1, 1)$, and $(p_1 + (1-p_2)\frac{a_1}{a_2}, 1)$, if the line $\ell$ of minimum slope passing through $B_1 \times B_2$ is such that $\hat{\ell}=a_1$ where $a=(a_1,a_2)$ is the (positive) unit direction vector of $\ell$.
\end{itemize}
Let us bound the area of the two triangles. Since the calculations are analogous, let us focus on $T_1(p)$.
By definition, the basis of $T_1(p)$ is smaller than $1$ while its height is bounded by $\frac{a_2}{a_1}$.
The maximum value for the height of $T_1(p)$ is achieved for $a_2=u^{\frac{1}{5}}$. So, by exploiting the identity $a_1^2+a_2^2=1$, we have
\begin{align*}
\left(\frac{a_2}{a_1}\right)^2 =\frac{u^{\frac{2}{5}}}{1-u^{\frac{2}{5}}}.
\end{align*}
Under the conditions $u\leq 2^{-\frac{5}{2}}$ and $\frac{1}{2}u^{-\frac{2}{5}}\geq 1$ we have
\begin{align*}
\frac{a_2}{a_1} \leq \sqrt{2} u^{\frac{1}{5}}.
\end{align*}
Therefore,
$\mathrm{area}(T_1(p))\leq \frac{ \sqrt{2} }{2} u^{\frac{1}{5}}.$
Similarly,
$\mathrm{area}(T_2(p))\leq \frac{ \sqrt{2} }{2} u^{\frac{1}{5}}.$
Finally, 
\begin{align*}
\vol(\boxset_{\text{non-diag}})&=\int_{\boxset_{\text{non-diag}}} 1 d\mu \\
&\leq\int_{p\in R} \int_{q\in T_1(p) \cup T_2(p) } 1 d\mu \\
&\leq\int_{p\in R} \sqrt{2} u^{\frac{1}{5}} d\mu \leq \sqrt{2} u^{\frac{1}{5}}. \qedhere
\end{align*}
\end{proof}

\paragraph{Good box pairs}
For good box pairs, we use the fact that the variation of a box pair yields a subinterval
of $[(v_\X-\delta_\X)(v_\Y-\delta_\Y)\vol(B_1 \times B_2),(v_\X+\delta_x)(v_\Y+\delta_\Y)\vol(B_1 \times B_2)]$ as an approximation, so
the width is bounded by $2(v_\X\delta_\Y+v_\Y\delta_\X)\vol(B_1 \times B_2)$. 
Let $\boxset_{\text{good}}$ be the union of all good box pairs.
Since
the volumes of all good box pairs sum up to at most one, that is, 
$\vol(\boxset_{\text{good}}) \leq 1$, 
 it follows that
the width of $J_{good}$ is bounded by $2(v_\X\delta_\Y+v_\Y\delta_\X)$.
By absolute boundedness, $v_\X$ and $v_\Y$ are in $O(n)$, and recall that by definition,
\[\delta_\X=\frac{v_3 n(2A+B)}{\sqrt{2}M\hat{{\ell_c}}}  +v_1n W + v_2 n L = O\left(n\left(\frac{A+B}{M^2}+W+L\right)\right)\]
based on the fact that $\hat{\ell}\geq M$. The same bound holds for $\delta_\Y$.
Hence,
\[\mathrm{width}(J_{good})=O\left(n^2\left(\frac{A+B}{M^2}+W+L\right)\right).\]
It remains to show that $\frac{A+B}{M^2}+W+L=O(u^{0.1})$.
Note that $M\geq u^{\frac{1}{5}}$ because the box pair is assumed
to be good. We will show in the next lemmas that $A$, $B$, $W$, and $L$
are all in $O(\sqrt{u})$, proving that the term is indeed in $O(u^{0.1})$.
This completes the proof of the complexity of the algorithm.

\begin{lemma}
\label{angle_lemma}
Let $(B_1,B_2)$ be a good box pair. Let $a$, $a'$ be the unit direction vectors of two lines that
pass through the box pair. Then, $\|a-a'\|_\infty \leq 2\sqrt{u}$. In particular, $A=O(\sqrt{u})$.
\end{lemma}
\begin{proof}
Since $(B_1,B_2)$ is a good box pair, the largest value for $\|a-a'\|_\infty$ is achieved when $\ell$ and $\ell'$ correspond to the lines passing through the box pair$(B_1,B_2)$ with minimum and maximum slope, respectively. By denoting as $c_1=(c_{1,x}, c_{1,y})$, $c_2=(c_{2,x}, c_{2,y})$ the centers of $B_1$, $B_2$, we define $\ell$ to be the line passing through the points $c_1+(-\frac{u}{2},\frac{u}{2})$, $c_2+(\frac{u}{2},-\frac{u}{2})$. Similarly, let us call $\ell'$ the line passing through the points $c_1+(\frac{u}{2},-\frac{u}{2})$, $c_2+(-\frac{u}{2},\frac{u}{2})$.
So, the unit direction vector $a$ of $\ell$ can be expressed as
\begin{align*}
a=\frac{(c_2+(\frac{u}{2},-\frac{u}{2}))-(c_1+(-\frac{u}{2},\frac{u}{2}))}{\norm{(c_2+(\frac{u}{2},-\frac{u}{2}))-(c_1+(-\frac{u}{2},\frac{u}{2}))}}.
\end{align*}

Similarly, the unit direction vector $a'$ of $\ell'$ is described by
\begin{align*}
a'=\frac{(c_2+(-\frac{u}{2},\frac{u}{2}))-(c_1+(\frac{u}{2},-\frac{u}{2}))}{\norm{(c_2+(-\frac{u}{2},\frac{u}{2}))-(c_1+(\frac{u}{2},-\frac{u}{2}))}}.
\end{align*}

Then, by denoting as $\langle \cdot, \cdot \rangle$ the scalar product,
\begin{align*}
  \|a-a'\|^2_\infty&\leq \|a-a'\|^2_2 = \|a\|^2_2 + \|a'\|^2_2 - 2 \langle a, a' \rangle = 2 (1- \langle a, a' \rangle)\\
  &=2\Big(1- \langle \frac{c_2-c_1+(u,-u)}{\norm{c_2-c_1+(u,-u)}}, \frac{c_2-c_1+(-u,u)}{\norm{c_2-c_1+(-u,u)}} \rangle\Big)\\
  &=2\Big(1-  \frac{\norm{c_2-c_1}^2 -2u^2 }{\norm{c_2-c_1+(u,-u)} \norm{c_2-c_1+(-u,u)}}\Big).
\end{align*}
By an elementary calculation, one can prove that
\begin{align*}
&\norm{c_2-c_1+(u,-u)} \norm{c_2-c_1+(-u,u)} \\ =&\sqrt{ 4u^2 \big(u^2 + 2 (c_{2,x}-c_{1,x})(c_{2,y}-c_{1,y})\big) + \norm{c_2-c_1}^4 }.
\end{align*}
Then,
\begin{align*}
&\|a-a'\|^2_\infty \\
\leq& 2\Big(1-  \frac{\norm{c_2-c_1}^2 -2u^2 }{ \sqrt{ 4u^2 \big(u^2 + 2 (c_{2,x}-c_{1,x})(c_{2,y}-c_{1,y})\big) + \norm{c_2-c_1}^4 } }\Big)\\
=&2\Big(1 +  \frac{2u^2 - \norm{c_2-c_1}^2}{ \sqrt{ 4u^2 \big(u^2 + 2 (c_{2,x}-c_{1,x})(c_{2,y}-c_{1,y})\big) + \norm{c_2-c_1}^4 } }\Big).
\end{align*}

Since $(B_1,B_2)$ is a good box pair, $\norm{c_2-c_1}\geq\sqrt{u}$. So,
\begin{align*}
&\|a-a'\|^2_\infty 
\leq 2\Big(1 +  \frac{2u - 1}{ \sqrt{ 4 \big(u^2 + 2 (c_{2,x}-c_{1,x})(c_{2,y}-c_{1,y})\big) + 1 } }\Big).
\end{align*}
Since $\sqrt{ 4 \big(u^2 + 2 (c_{2,x}-c_{1,x})(c_{2,y}-c_{1,y})\big) + 1 }\geq 1$, we have that
\begin{align*}
\|a-a'\|^2_\infty&\leq 2(1 + 2u - 1)=4u.
\end{align*}
Therefore, 
\begin{align*}
\|a-a'\|_\infty&\leq 2\sqrt{u}. \qedhere
\end{align*}
\end{proof}

\begin{lemma}
\label{b_lemma}
Let $(B_1,B_2)$ be a good box pair. Let $\ell=a\lambda+b$, $\ell'=a'\lambda+b'$ be two lines that pass
through the box pair such that $a$, $a'$ are unit direction vectors and $b$, $b'$ are the intersection points with the diagonal of the second and the fourth quadrant. Then $\|b-b'\|_\infty \leq 4 \sqrt{u}$. In particular, $B=O(\sqrt{u})$.
\end{lemma}

\begin{proof} Since $(B_1,B_2)$ is a good box pair, the largest value for $\|b-b'\|_\infty$ is achieved when $\ell$ and $\ell'$ correspond to the lines passing through the box pair$(B_1,B_2)$ with minimum and maximum slope, respectively. By denoting the centers of $B_1$ and $B_2$ by $c_1$ and $c_2$, we define $\ell$ to be the line passing through the points $c_1+(-\frac{u}{2},\frac{u}{2})$, $c_2+(\frac{u}{2},-\frac{u}{2})$. Similarly, let us call $\ell'$ the line passing through the points $c_1+(\frac{u}{2},-\frac{u}{2})$, $c_2+(-\frac{u}{2},\frac{u}{2})$.
So, $\ell$ can be expressed as
\begin{align*}
(x,y)=\frac{c_2+(\frac{u}{2},-\frac{u}{2})-c_1-(-\frac{u}{2},\frac{u}{2})}{\norm{c_2+(\frac{u}{2},-\frac{u}{2})-c_1-(-\frac{u}{2},\frac{u}{2})}}t+c_1+(-\frac{u}{2},\frac{u}{2}),
\end{align*}
where $t$ is a parameter running on $\R$.
By intersecting $\ell$ with the line $y=-x$, we get:
\begin{align*}
&\frac{c_{2,x}+\frac{u}{2}-c_{1,x}+\frac{u}{2}}{\norm{c_2+(\frac{u}{2},-\frac{u}{2})-c_1-(-\frac{u}{2},\frac{u}{2})}} t+c_{1,x}-\frac{u}{2} \\
=&\frac{-c_{2,y}+\frac{u}{2}+c_{1,y}+\frac{u}{2}}{\norm{c_2+(\frac{u}{2},-\frac{u}{2})-c_1-(-\frac{u}{2},\frac{u}{2})}}t-c_{1,y}-\frac{u}{2},
\end{align*}
which can be written as
\begin{align*}
c_{1,x}+c_{1,y}=\frac{c_{1,x}+c_{1,y}-c_{2,x}-c_{2,y}}{\norm{c_2+(\frac{u}{2},-\frac{u}{2})-c_1-(-\frac{u}{2},\frac{u}{2})}}t,
\end{align*}
letting us deduce that
\begin{align*}
t=&\frac{(c_{1,x}+c_{1,y})\norm{c_2+(\frac{u}{2},-\frac{u}{2})-c_1-(-\frac{u}{2},\frac{u}{2})}}{c_{1,x}+c_{1,y}-c_{2,x}-c_{2,y}}.
\end{align*}
So, by replacing $t$ in the equation of $\ell$ we retrieve $b$:
\begin{align*}
b&=\frac{c_2+(\frac{u}{2},-\frac{u}{2})-c_1-(-\frac{u}{2},\frac{u}{2})}{\norm{c_2+(\frac{u}{2},-\frac{u}{2})-c_1-(-\frac{u}{2},\frac{u}{2})}}\\
&\frac{(c_{1,x}+c_{1,y})\norm{c_2+(\frac{u}{2},-\frac{u}{2})-c_1-(-\frac{u}{2},\frac{u}{2})}}{c_{1,x}+c_{1,y}-c_{2,x}-c_{2,y}}+c_1+(-\frac{u}{2},\frac{u}{2})\\
&=\frac{(u,-u)(c_{1,x}+c_{1,y})}{c_{1,x}+c_{1,y}-c_{2,x}-c_{2,y}}+\frac{(c_2-c_1)(c_{1,x}+c_{1,y})}{c_{1,x}+c_{1,y}-c_{2,x}-c_{2,y}}\\
&+c_1+(-\frac{u}{2},\frac{u}{2}).
\end{align*}
Similarly,
\begin{align*}
b'=&\frac{(-u,u)(c_{1,x}+c_{1,y})}{c_{1,x}+c_{1,y}-c_{2,x}-c_{2,y}}+\frac{(c_2-c_1)(c_{1,x}+c_{1,y})}{c_{1,x}+c_{1,y}-c_{2,x}-c_{2,y}}\\
&+c_1+(\frac{u}{2},-\frac{u}{2}).
\end{align*}
So,
\begin{align*}
\infnorm{b-b'}&=\infnorm{\Big(2\frac{c_{1,x}+c_{1,y}}{c_{1,x}+c_{1,y}-c_{2,x}-c_{2,y}}-1\Big)(u,-u)}\\
&=\Big|\frac{c_{1,x}+c_{1,y}+c_{2,x}+c_{2,y}}{c_{2,x}+c_{2,y}-c_{1,x}-c_{1,y}}\Big|\infnorm{(u,-u)}\\
&\leq\frac{4r}{|c_{2,x}+c_{2,y}-c_{1,x}-c_{1,y}|}u.
\end{align*}
Since $(B_1,B_2)$ is a good box pair,
\begin{align*}
c_{2,x}+c_{2,y}-c_{1,x}-c_{1,y} = \left\lVert c_2-c_1 \right\rVert_1 \geq \norm{c_2-c_1} \geq \sqrt{u}.
\end{align*}
Finally,
\begin{align*}
\infnorm{b-b'}&\leq \frac{4}{\sqrt{u}}u = 4 \sqrt{u}. \qedhere
\end{align*}
\end{proof}

\begin{lemma}
\label{weight_lemma}
Let $(B_1,B_2)$ be a good box pair. Let $\hat{\ell}$, $\hat{\ell'}$ be the weights of two lines $\ell$ and $\ell'$ that
pass through the box pair. Then $|\hat{\ell}-\hat{\ell'}| \leq 4\sqrt{u}$. In particular, $W=O(\sqrt{u})$.
\end{lemma}
\begin{proof}
If $\hat{\ell}=a_1$ and $\hat{\ell'}=a'_1$, then, by applying Lemma \ref{angle_lemma},
\begin{align*}
|\hat{\ell}-\hat{\ell'}|=|a_1-a'_1|\leq \infnorm{a-a'}\leq 2 \sqrt{u}.
\end{align*}
On the other hand, if $\hat{\ell}=a_1$ and $\hat{\ell'}=a'_2$, then there exists a line $\ell''$ passing through the box pair $(B_1, B_2)$ such that $a''=(\frac{\sqrt{2}}{2}, \frac{\sqrt{2}}{2})$. By applying twice Lemma \ref{angle_lemma},
\begin{align*}
|\hat{\ell}-\hat{\ell'}|&=|a_1-a'_2|\leq |a_1- \frac{\sqrt{2}}{2}| + | \frac{\sqrt{2}}{2}-a'_2| \\
&= |a_1- a''_1| + | a''_2-a'_2| \leq \infnorm{a-a''} + \infnorm{a''-a'} \\
&\leq 4 \sqrt{u}.
\end{align*}
The cases $\hat{\ell}=a_2$, $\hat{\ell'}=a'_2$ and $\hat{\ell}=a_2$, $\hat{\ell'}=a'_1$ can be treated analogously to the previous ones.
\end{proof}

\begin{lemma}
\label{lambda_lemma}
Let $(p,q)$, $(p',q')$ be two points in a good box pair $(B_1,B_2)$ and let $\ell$, $\ell'$ be the lines passing through $p$, $q$ and $p'$, $q'$, respectively. In accordance with the usual parametrization, we have that $|\lambda_p-\lambda_{p'}| \leq \sqrt{2}u + 4 \sqrt{u}$ and $|\lambda_q-\lambda_{q'}| \leq \sqrt{2}u + 4 \sqrt{u}$. As a consequence, $L=O(\sqrt{u})$.
\end{lemma}

\begin{proof}
Thanks to the definition of $\lambda_p$, the triangular inequality and Lemma \ref{b_lemma}, we have that:
\begin{align*}
\lambda_p&=\norm{p-b} \leq \norm{p-p'} + \norm{p'-b'} + \norm{b' - b}\\
&\leq \sqrt{2}u + \lambda_{p'} + 4 \sqrt{u}.
\end{align*}
So, we have that $\lambda_p - \lambda_{p'} \leq \sqrt{2}u + 4  \sqrt{u},$
and, similarly, $\lambda_{p'} - \lambda_{p} \leq \sqrt{2}u + 4  \sqrt{u}.$
Then,
$$|\lambda_p-\lambda_{p'}| \leq \sqrt{2}u + 4  \sqrt{u}.$$
Analogously, it can be proven that
\begin{align*}
|\lambda_q-\lambda_{q'}| &\leq \sqrt{2}u + 4  \sqrt{u}. \qedhere
\end{align*}
\end{proof}

\newpage{} 
\end{document}